\documentclass[final]{cvpr}

\usepackage{times}
\usepackage{epsfig}
\usepackage{graphicx}
\usepackage{amsmath, amsfonts, amsthm}
\usepackage{amssymb}
\usepackage{comment}
\usepackage{caption}
\usepackage{subcaption}
\usepackage{multirow}
\usepackage{multicol}
\usepackage{enumitem}

\newcommand{\ZB}{\mathbb{Z}}
\newcommand{\RB}{\mathbb{R}}

\newcommand{\gf}{\mathfrak{g}}

\newcommand{\ad}{ad}

\newcommand{\Aff}{Aff}
\newcommand{\Hom}{Hom}
\newcommand{\aff}{\mathfrak{aff}}
\newcommand{\hm}{\mathfrak{hom}}

\newtheorem{thm}{Theorem}[section]

\newtheorem{lemma}[thm]{Lemma}
\newtheorem{prop}[thm]{Theorem}

\usepackage[normalem]{ulem}
\usepackage[pagebackref=true,breaklinks=true,colorlinks,bookmarks=false]{hyperref}

\def\cvprPaperID{10467} % *** Enter the CVPR Paper ID here
\def\confYear{CVPR 2021}

\begin{document}

%%%%%%%%% TITLE
\title{Enabling Equivariance for Arbitrary Lie Groups}

\author{Lachlan E. MacDonald\thanks{\tt\small lachlan.macdonald@adelaide.edu.au},\, Sameera Ramasinghe,\, Simon Lucey\\
Australian Institute for Machine Learning\\
University of Adelaide, Australia
}
\maketitle

%%%%%%%%% ABSTRACT
\begin{abstract}
Although provably robust to translational perturbations, convolutional neural networks (CNNs) are known to suffer from extreme performance degradation when presented at test time with more general geometric transformations of inputs.  Recently, this limitation has motivated a shift in focus from CNNs to Capsule Networks (CapsNets).  However, CapsNets suffer from admitting relatively few theoretical guarantees of invariance.  We introduce a rigourous mathematical framework to permit invariance to any Lie group of warps, exclusively using convolutions (over Lie groups), without the need for capsules.  Previous work on group convolutions has been hampered by strong assumptions about the group, which precludes the application of such techniques to common warps in computer vision such as affine and homographic.  Our framework enables the implementation of group convolutions over \textbf{any} finite-dimensional Lie group. We empirically validate our approach on the benchmark affine-invariant classification task, where we achieve $\sim$30\% improvement in accuracy against conventional CNNs while outperforming most CapsNets. As further illustration of the generality of our framework, we train a homography-convolutional model which achieves superior robustness on a homography-perturbed dataset, where CapsNet results degrade.
\end{abstract}

%%%%%%%%% BODY TEXT
\section{Introduction}

Symmetry (reversible, composable change) is ubiquitous in the visual reality perceived by humans. It is therefore no surprise that symmetry now occupies a fundamental and increasingly well-understood role in machine learning, particularly in computer vision \cite{gpequi, gaugeequi, steerable1, steerable2, steerable3, steerable4, surjexp, sameeraequi, gdl, coordind}.  \emph{Invariance} with respect to symmetries is particularly desirable in machine learning. Presently this is frequently achieved by training models on large volumes of symmetry-perturbed data.  However, where there is a paucity of data or a need for explainability, architectural approaches such as CapsNets may be preferable \cite{capsule1, capsule2,matrixcaps}.

\begin{figure}
    \centering
    \includegraphics[scale=0.8]{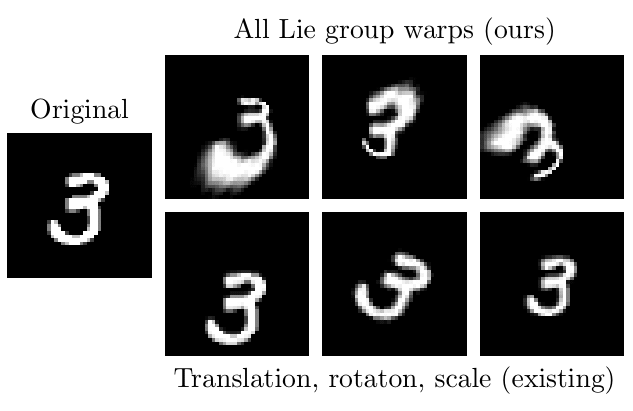}
    \caption{Domains of applicability of our framework and existing convolutional frameworks.}
    \label{fig:my_label}
\end{figure}

CapsNets admit very few theoretical guarantees of invariance \cite{gecaps}.  However they are not the only means of achieving architectural invariance.  Provable architectural invariance to \emph{translational symmetries} is already pervasive in computer vision in the form of CNNs.  The convolutions implemented in modern CNNs are, by definition, sums over responses between filter and feature, as the filter is translated across different positions.  Crucially, these operations are \emph{equivariant} in the sense that translating the input feature and then convolving is the same as first convolving and then translating the response.  Carefully designed downsampling operations \cite{shiftinv} turn this equivariance into \emph{invariance}, so that translation of the input signal does not change the response.  The respect of CNNs for the spatial locality inherent in natural images, together with their equi/in-variance, constitute the inductive biases that have seen CNNs become the industry standard in many computer vision tasks.

More recently still, in the seminal paper by Cohen et. al. \cite{gpequi}, it was observed that the translational convolutions of CNNs could be replaced by more general \emph{group convolutions} that have for decades played an important role in mathematics, particularly in operator theory and noncommutative geometry \cite{ncg, holkar}. Given a group $G$ of warps, a $G$-convolution is a sum over \emph{warps} belonging to $G$ of responses between a warped filter and feature; when $G$ is the group of translations, one recovers ordinary convolution as used in conventional CNNs.  When built into neural networks, these group convolutions achieve better weight sharing and superior performance on image classification tasks than their conventional cousins, especially on datasets consisting of $G$-warped data \cite{gpequi, gaugeequi, steerable1, steerable2, steerable3, steerable4, surjexp, sameeraequi, gdl, coordind}.

One aspect of group-convolutional neural networks which has received relatively little attention is their capacity for a mild form of out-of-distribution generalisation.  Just as well-designed conventional CNNs perform well when tested on translationally perturbed data, a network with $G$-convolutional layers and appropriate downsampling will theoretically perform well on $G$-perturbed data.  As far as the authors are aware, the only place this problem has been studied is in \cite{sphericalcnns}, where spherical CNNs (in which $G$ is the special orthogonal group) are trained on data on a sphere and tested on rotated data with excellent results.  One of the primary reasons for this scarcity of study appears to be the lack of a sufficiently general theoretical framework. While $G$-convolutions can be efficiently implemented for finite groups \cite{gpequi} and a small family of continuous (Lie) groups \cite{sphericalcnns, gaugeequi, surjexp}, current techniques do not allow for the implementation of $G$-convolutions for Lie groups that are common in natural images, such as the affine and homography groups.  Should this hurdle be overcome, $G$-convolutions could constitute an effective tool for common computer vision problems which circumvent the need for data augmentation and huge numbers of parameters.

\noindent\textbf{Contributions:}  First, we introduce a rigourous mathematical framework using tools from topology to enable well-defined convolutional networks over \emph{arbitrary} finite-dimensional Lie groups.  Up until now, convolutional networks could only be implemented over finite groups or Lie groups with surjective exponential map, ruling out many natural warps.  Second, we introduce tools from differential geometry and Lie theory to enable Markov Chain Monte-Carlo (MCMC) sampling from Haar measure for arbitrary Lie groups.  This sampling may be done in a scalable and easily parallelisable fashion, permitting fast approximation of group convolutions.  Sampling from Haar measure had previously not been possible beyond a small family of Lie groups \cite{randommatrices, surjexp}.  Third and finally, we illustrate the validity of our mathematical framework by testing it on the benchmark affine-invariant image classification task, where it outperforms the SOTA $E(2)$-equivariant CNN \cite{steerable4}, and all but one of the existing CapsNet benchmarks.  Finally, we demonstrate the generality of our theory by establishing a new benchmark on homography-invariant image classification, which to the best of our knowledge has not previously been studied.  On this benchmark, our method exhibits stable performance, while that of the CapsNets for which we could obtain code \cite{capsule2, gecaps} degrades\footnote{The code we used for our experiments is available at https://github.com/lemacdonald/equivariant-convolutions}.

\noindent\textbf{Acknowledgements:} We thank the reviewers for their substantive critiques, which have helped to improve the paper.

\section{Related works}

Group equivariant convolutional networks were first introduced in \cite{gpequi}, which considered finite groups only. The theory and methods presented therein were subsequently refined in \cite{steerable1, kondortrivedi}, and generalised to various subgroups of the Euclidean (rigid motion) groups in \cite{sphericalcnns, steerable1, steerable2, steerable3, steerable4}.  Additional equivariance to scaling transformations was demonstrated in \cite{warpedconv, scaleequi}, and equivariant attention using convolutions introduced in \cite{attentionequi}.  In \cite{equimlp}, a method is given for the construction of multilayer perceptrons that are equivariant to any matrix group.  This method uses (non-convolutional) finite-dimensional representations rather than the infinite-dimensional regular representation, and as demonstrated in \cite[Section 4]{steerable1}, there are situations where the latter has advantages in accuracy over the former.  Prior to the present paper, the most general method in the literature which uses the infinite-dimensional regular representation of a group is given by Finzi et. al. \cite{surjexp}.  The approach of Finzi et. al. applies to all Lie groups with surjective exponential map, and via this assumption parameterises filters on a (non-Euclidean) group as functions on the (Euclidean) Lie algebra.  Of all of these works, our own method is most closely related to the latter.  However, where as \cite{surjexp} applies only to Lie groups with surjective exponential map, ruling out for instance the affine and homography groups, our method applies to all finite-dimensional Lie groups.  

Capsules were introduced in \cite{capsule1} as a way of organising neurons so as to better deal with expected variations in input such as position, scale, lighting and orientation.  In \cite{capsule2}, CapsNets were introduced as a way of integrating capsules into CNN architectures. It was shown that using a dynamic routing procedure enabled CapsNets to achieve 79\% test accuracy on the affNIST\footnote{http://www.cs.toronto.edu/~tijmen/affNIST/} (affine-perturbed MNIST) test set, after being trained on unperturbed MNIST.  This is a significant improvement on the test accuracy of 66\% achieved by conventional CNNs.  The task of training on unperturbed MNIST and testing on affNIST has since become a benchmark for affine-invariant image classification \cite{capsule2, sparsecaps, gecaps, affcapsnets, routinguncertainty}.  The present state of the art on affNIST is one of the models given in \cite{routinguncertainty}, which introduces a variational routing procedure to give a test performance of 97.69\% on affNIST after training on MNIST.  Our method outperforms all CapsNet benchmarks bar this state of the art, including the smaller of the models given in \cite{routinguncertainty}.  Precisely how CapsNets achieve such remarkable results, without admitting any theoretical guarantees of out-of-distribution generalisation, is not presently understood.

\section{Preliminaries}

Recall that a \emph{group} is a set $G$ equipped with an associative multiplication $G\times G\ni(u_{1},u_{2})\mapsto u_{1}u_{2}\in G$, a multiplicative identity $e$ such that $ue = eu = u$ for all $u\in G$, and an inversion $G\ni u\mapsto u^{-1}\in G$ for which $uu^{-1} = u^{-1}u = e$ and such that $(u_{1}u_{2})^{-1} = u_{2}^{-1}u_{1}^{-1}$ for all $u,u_{1},u_{2}\in G$.  A group $G$ is said to \emph{act} on a set $X$ if it admits a multiplication $G\times X\ni(u,x)\mapsto u\cdot x\in X$ such that $(u_{1}u_{2})\cdot x = u_{1}\cdot(u_{2}\cdot x)$ for all $u_{1},u_{2}\in G,\,x\in X$.  If $G$ acts on sets $X$ and $Y$, with actions denoted by $\cdot_{X}$ and $\cdot_{Y}$ respectively, then $\phi:X\rightarrow Y$ is said to be \emph{equivariant} if $\phi(g\cdot_{X} x) = g\cdot_{Y}\phi(x)$ for all $g\in G$ and $x\in X$.  An important special case of equivariance is \emph{invariance}, which is when $G$ acts trivially on $Y$; in this case, one has $\phi(g\cdot x) = \phi(x)$ for all $g\in G$ and $x\in X$.  Building a network that is robust to $G$-perturbations of its input amounts to building a network which is \emph{invariant} to $G$ in the above sense.  It is easy to see that if $\phi_{1}$ is equivariant and $\phi_{2}$ is invariant, then $\phi_{2}\circ\phi_{1}$ is invariant -- in this way, equivariance in intermediate layers of a network facilitates invariance of the network as a whole. 

We also recall that a group $G$ is a \emph{Lie group} if it is simultaneously a smooth manifold with respect to which the multiplication and inversion are smooth maps.  If $G$ is a Lie group, there is a vector space $\gf$ called its \emph{Lie algebra}, equipped with an antisymmetric multiplication $[\cdot,\cdot]:\gf\times\gf\rightarrow\gf$ called the \emph{Lie bracket}, and an \emph{exponential map} $\exp:\gf\rightarrow G$, which is one-to-one from a neighbourhood of zero in $\gf$ onto a neighbourhood of the identity in $G$.  We will be particularly concerned with \emph{matrix groups}, which are groups of invertible matrices.  Any matrix group is a Lie group, with multiplication and inversion given by the usual matrix operations, and with Lie algebra given by some space of not-necessarily-invertible matrices on which the Lie bracket is defined by the matrix commutator:
\[
[A,B]:=AB - BA.
\]
For matrix groups, the exponential map is given by the matrix exponential.  In what remains of this section, the reader may restrict to matrix groups, although the theory is valid in full generality.

Convolutions are defined in terms of integrals over Lie groups $G$.  Recall that any Lie group $G$ admits the \emph{left (resp. right) Haar measure} $\mu_{L}$ (resp. $\mu_{R}$) \cite[Section 1.3]{williams}, with respect to which integrals may be performed. Left Haar measure is \emph{left-invariant} in the sense that $\int_{G}f(vu)d\mu_{L}(u) = \int_{G}f(u)d\mu_{L}(u)$ for any $v\in G$ and any integrable function $f$.  Similarly, right Haar measure is \emph{right-invariant} in the sense that $\int_{G}f(uv)d\mu_{R}(u) = \int_{G}f(u)d\mu_{R}(u)$ for all $v\in G$ and integrable functions $f$. These measures are related by $d\mu_{L}(u^{-1}) = d\mu_{R}(u)$.  On the zero-dimensional Lie group $\ZB^{2}$ of integer translations, for instance, Haar measure is simply counting measure -- the Haar measure of any subset of $\ZB^{2}$ is simply the number of points in that subset.

We denote by $C(G;\RB^{K})$ the set of continuous functions $G\rightarrow\RB^{K}$ on a Lie group $G$.  A function $\psi\in C(G;\RB^{K})$ is said to be \emph{compactly supported} if, roughly speaking, it is zero outside a set of finite Haar measure\footnote{Strictly speaking, a $\psi$ being ``compactly supported" means that the topological closure of the set on which $\psi$ is nonzero is topologically compact.  Since Haar measures are Radon measures \cite[Section 1.3]{williams}, compact sets do have finite measure, but the converse is not true in general.}.  We denote by $C_{c}(G;\RB^{K})\subset C(G;\RB^{K})$ the set of compactly supported continuous functions.  A function $\psi$ being compactly supported means in particular that the integrals $\int_{G}\psi(u) d\mu_{L}(u)$ and $\int_{G}\psi(u) d\mu_{R}(u)$ are finite.  As a consequence, letting $\cdot$ denote matrix multiplication, the \emph{convolution} $f*\psi$ of $f\in C_{c}(G;\RB^{K})$ and $\psi\in C_{c}(G;\RB^{K\times L})$ defined by
\begin{align}
f*\psi(u) =& \int_{G}f(v)\cdot\psi(v^{-1}u)d\mu_{L}(v)\label{compconv1} \\ =& \int_{G}f(uv^{-1})\cdot\psi(v)d\mu_{R}(v)\label{compconv2}
\end{align}
is itself a compactly supported and continuous function \cite[Proposition 1.1]{renault}.  The convolution formula of Kondor-Trivedi \cite{kondortrivedi}, who work with \emph{compact} groups, is a special case of this, as is the convolution formula for conventional CNNs
\[
f*\psi(m,n) = \sum_{(i,j)\in\ZB^{2}}f(i,j)\cdot\psi(m-i,n-j)
\]
when $G = \ZB^{2}$.  For $v\in G$ the map $L_{v}:C(G;\RB^{K})\rightarrow C(G;\RB^{K})$ defined by $L_{v}(f)(u):=f(v^{-1}u)$ preserves the subset $C_{c}(G;\RB^{K})$, and satisfies $L_{v}(f)*\psi = L_{v}(f*\psi)$ for all $f\in C(G;\RB^{K})$ and $\psi\in C_{c}(G;\RB^{K\times L})$.  That is, convolutions are \emph{equivariant} to warps of feature maps $f$ by warps in $G$.  It is this equivariance that accounts for the robust performance of $G$-CNNs \cite{gpequi}.  As we will describe in the next section, when working with the affine and homography groups we do not have the luxury of assuming our feature maps to be compactly supported -- in fact, our feature maps must be allowed to be \emph{unbounded} in that they may blow up to infinity away from the identity.  One of our theoretical contributions is showing that convolution is still well-defined in such a case.

An important theorem from Lie group theory is the Schur-Poincar\'{e} formula for the derivative of the exponential map. This formula is the key to our theorem that enables sampling from Haar measure for arbitrary Lie groups. 

\begin{thm}\cite[Theorem 5, Section 1.2]{rossman}\label{dexp}
    Let $\gf$ be the Lie algebra of a finite-dimensional Lie group $G$, and let $t\mapsto\xi(t)$ be a curve in $\gf$.  Then one has
    \[
    \frac{d}{dt}\bigg|_{t=0}\exp(X(t)) = dL_{\exp(X(0))}\,\frac{1-e^{-\ad_{X(0)}}}{\ad_{X(0)}}\frac{d}{dt}\bigg|_{t=0}X(t),
    \]
    where for $X\in\gf$, $\ad_{X}:\gf\rightarrow\gf$ is the linear map defined by $\ad_{X}(Y):=[X,Y]$, $dL_{\exp(X)}$ denotes the derivative of left-multiplication by $\exp(X)\in G$, and
    \[
    \frac{1-e^{-\ad_{X}}}{\ad_{X}} = \sum_{k=0}^{\infty}\frac{(-1)^{k}}{(k+1)!}(\ad_{X})^{k}
    \]
    is a power series in the linear map $\ad_{X}$.\qed
\end{thm}

The Lie groups of experimental interest in this paper are the \emph{affine group} $\Aff$ and the \emph{homography group} $\Hom$, both of which are ubiquitous in computer vision.  The affine group $\Aff$, generated by all shears, scalings, rotations and translations, is the set of all invertible $3\times 3$ matrices of the form
\begin{equation}\label{eqp1}
    \begin{pmatrix} u_{1} & u_{2} & u_{3} \\ u_{4} & u_{5} & u_{6} \\ 0 & 0 & 1\end{pmatrix},
\end{equation}
The Lie algebra $\aff$ of $\Aff$ is the space of $3\times3$ matrices with bottom row equal to zero.  We choose as an ordered basis the $3\times3$ matrices given by the outer products $\{e_{1}e_{1}^{T}, e_{1}e_{2}^{T}, e_{1}e_{3}^{T}, e_{2}e_{1}^{T}, e_{2}e_{2}^{T},e_{2}e_{3}^{T}\}$, where $\{e_{i}\}_{i=1}^{3}$ is the canonical basis for $\RB^{3}$.  Taken in isolation, the first and fifth of these basis elements determine scaling, the second and fourth determine shearing, and the third and sixth determine translation, although due to the noncommutativity of the affine group this description fails for nontrivial linear combinations of the basis elements.  The affine group \emph{acts} on the space $\RB^{2}$ by the formula
\begin{equation}\label{eqp2}
    \begin{pmatrix} u_{1} & u_{2} & u_{3} \\ u_{4} & u_{5} & u_{6} \\ 0 & 0 & 1\end{pmatrix}\cdot\begin{pmatrix} x_{1} \\ x_{2}\end{pmatrix}:= \begin{pmatrix} u_{1}x_{1} + u_{2}x_{2} + u_{3} \\ u_{4}x_{1} + u_{5}x_{2} + u_{6}\end{pmatrix}.
\end{equation}

The six-dimensional affine group may be identified as a subgroup of the eight-dimensional \emph{homography group} $\Hom$, which is generated by the perspective projections together with all affine transformations and may be identified with the group of $3\times3$ matrices with determinant one \cite{homsl}.  The Lie algebra $\hm$ of the homography group is then the space of $3\times 3$ matrices with zero trace, and as an ordered basis we choose the matrices $\{e_{1}e_{1}^{T}-(1/3)I, e_{1}e_{2}^{T}, e_{1}e_{3}^{T}, e_{2}e_{1}^{T}, e_{2}e_{2}^{T}-(1/3)I, e_{2}e_{3}^{T}, e_{3}e_{1}^{T}, e_{3}e_{2}^{T}\}$.  The linear map $\phi:\aff\rightarrow\hm$ defined by sending the aforementioned basis of $\aff$ to the first six basis elements of $\hm$ is an inclusion of Lie algebras, so that the first six basis elements of $\hm$ can be interpreted as scaling, shearing and translational warps just as for the affine group, while the final two basis elements determine perspective projections.  The corresponding inclusion $\Phi:\Aff\rightarrow\Hom$ of Lie groups is given by the formula $\Phi(A):=\det(A)^{-\frac{1}{3}}A$, and one has $\exp\circ\phi = \Phi\circ\exp$.

Having described the necessary fundamentals on Lie groups and convolutions, in the next section we will describe our theoretical contributions which allow convolutions over arbitrary Lie groups to be built into neural networks.  In the following section, we will empirically validate our theoretical framework using the affine and homography groups.

\section{Theoretical framework}

In this section we lay out our theoretical contributions.  We also describe how they may be implemented.

\subsection{The first layer}\label{firstlayer}

One key problem that must be overcome in using group convolutions for image classification is that the feature maps one is typically given as data are functions on $\ZB^{2}$ (giving intensity values to each pixel coordinate) rather than functions on a group $G$.  One must therefore ``lift" functions on $\ZB^{2}$ to functions on $G$.  An elegant solution proposed by Cohen et. al. \cite{gpequi} is to use a modified convolutional layer.  Suppose that $G$ is a finite group acting on $\ZB^{2}$, with the action of $u\in G$ on $(n,m)\in\ZB^{2}$ denoted $u(n,m)$.  Then given a feature map $f\in C_{c}(\ZB^{2};\RB^{K})$ and filter $\psi\in C_{c}(\ZB^{2};\RB^{K\times L})$, one defines $f*\psi\in C_{c}(G;\RB^{L})$ by the formula
\begin{equation}\label{eq4}
    f*\psi(u):=\sum_{(n,m)\in\ZB^{2}}f(n,m)\cdot\psi(u^{-1}(n, m)).
\end{equation}
It is proved in \cite[Equation 12]{gpequi} that Equation \eqref{eq4} is equivariant in the sense that $L_{v}(f)*\psi = L_{v}(f*\psi)$, where $L_{v}(f)(n,m):=f(v^{-1}(n,m))$.  Thus one is able to convert feature maps on $\ZB^{2}$ into feature maps on $G$ in an equivariant fashion, and these can then be fed into higher convolutional layers given by Equations \eqref{compconv1} and \eqref{compconv2}.

Equation \eqref{eq4} was also identified by mathematicians at around the same time as a kind of generalised Hilbert space inner product \cite{holkar}.  Note that $f*\psi$ is compactly supported by virtue of the fact that the action of $\ZB^{2}$ on itself is a proper\footnote{An action of a topological group $G$ on a topological space $X$ is \emph{proper} if for any compact set $K$ in $X\times X$, the preimage of $K$ under the map $G\times X\ni (u,x)\mapsto(u\cdot x,x)\in X\times X$ is compact} action \cite{holkar}.  The actions of the affine and homography groups on the plane are \emph{not} proper, and as we will soon see this means our theory for higher hidden layers needs to be sufficiently flexible to allow for \emph{unbounded} feature maps, which have not previously been considered in the literature.

Supposing now that $G$ is a Lie group acting on $\RB^{2}$, Equation \eqref{eq4} ceases to work for two primary reasons.  Firstly, $G$ need not map $\ZB^{2}$ to itself (consider scaling by positive real numbers).  This can be overcome by replacing $\ZB^{2}$ by $\RB^{2}$, and defining $\psi$ by an MLP as in \cite{pointconv}.  The sum of Equation \eqref{eq4} can then be thought of as an approximation to the integral
\begin{equation}\label{eq5}
    f*\psi(u):=\int_{\RB^{2}}f(x)\cdot\psi(u^{-1}x)dx
\end{equation}
over $\RB^{2}$.  There is, however, a deeper problem: namely that Equation \eqref{eq5} is \emph{not equivariant} under any transformations which distort the volume of Euclidean space.

\begin{prop}\label{prop1}
    For $v\in G$ and $x\in \RB^{2}$, let $Dv(x)$ denote the Jacobian at $x$ of the map $\RB^{2}\rightarrow\RB^{2}$ defined by $v$.  Then one has $L_{v}(f)*\psi = L_{v}(f*\psi)$ for all $f\in C_{c}(\RB^{2};\RB^{K})$ and $\psi\in C_{c}(\RB^{2};\RB^{K\times L})$ if and only if $\det(Dv(x))=1$ for all $x\in\RB^{2}$.
\end{prop}

\begin{proof}
Fixing $v$, $f$ and $\psi$, one computes
\begin{align*}
    L_{v}(f)*\psi(u) =& \int_{\RB^{2}}f(v^{-1}x)\cdot\psi(u^{-1}x)dx \\ =& \int_{\RB^{2}}f(\tilde{x})\cdot \psi(u^{-1}v\tilde{x})d(v\tilde{x}) \\ =& \int_{\RB^{2}}f(\tilde{x})\cdot\psi((v^{-1}u)\tilde{x})\det(Dv(\tilde{x}))d\tilde{x},
\end{align*}
using the substitution $x = v\tilde{x}$ for the second line and the usual change of coordinate formula $d(\varphi(x)) = \det(D\varphi(x))dx$ for the third.  The result follows.
\end{proof}

Since the affine and homographic transformations we wish to consider for applications to computer vision need not preserve the Euclidean volume element in general, Theorem \ref{prop1} obstructs the naive use of Equation \eqref{eq5}.  We propose instead the following modification, which nonetheless reduces to Equation \eqref{eq5} for volume-preserving warps.  Similar formulae appear in less general settings in \cite{warpedconv, scaleequi}.

\begin{thm}\label{firstlayerconv}
Define the $G$-convolution of $f\in C_{c}(\RB^{2};\RB^{K})$ and $\psi\in C_{c}(\RB^{2};\RB^{K\times L})$ by
\begin{equation}\label{eq6}
    f*\psi(u):=\int_{\RB^{2}}f(x)\cdot\psi(u^{-1}x)\det(Du^{-1}(x))dx.
\end{equation}
Then $f*\psi\in C(G;\RB^{L})$, and for all $v\in G$ one has $L_{v}(f)*\psi = L_{v}(f*\psi)$.
\end{thm}

\begin{proof}
We defer a proof of the fact that $f*\psi$ is a continuous function to the supplementary material, and here prove only equivariance.  Fixing $v$, $f$ and $\psi$ we see that $L_{v}(f)*\psi(u)$ is given by
\begin{align*}
&\int_{\RB^{2}}f(v^{-1}x)\cdot\psi(u^{-1}x)\det(Du^{-1}(x))dx \\ =& \int_{\RB^{2}}f(\tilde{x})\cdot\psi(u^{-1}v\tilde{x})\det(Du^{-1}(v\tilde{x}))\det(Dv(\tilde{x}))d\tilde{x} \\ =& \int_{\RB^{2}}f(\tilde{x})\cdot\psi((v^{-1}u)^{-1}\tilde{x})\det(D(u^{-1}v)(\tilde{x}))d\tilde{x} \\ =& \int_{\RB^{2}}f(\tilde{x})\cdot\psi((v^{-1}u)^{-1}\tilde{x})\det(D(v^{-1}u)^{-1}(\tilde{x}))d\tilde{x} \\ =& L_{v}(f*\psi)(u),
\end{align*}
where for the second line we have made the substitution $x = v\tilde{x}$, and for the third we have applied the chain rule.
\end{proof}

In practice, we approximate Equation \eqref{eq6} by a sum over the discrete integer lattice in $\RB^{2}$.  Interpolation of the image $f$, as is required in \cite{warpedconv}, is therefore unnecessary.  We note that the ``lifting" of a feature pair $(x,f(x))$ implicit in Equation \eqref{eq6} to a new pair $(u, (f*\psi)(u))$ is achieved via an entirely different means in \cite{surjexp}, where instead a pair $(x,f(x))$ is lifted to a triplet $(u,q,f(x))$,  with $u$ a group element and $q$ is an element of the orbit of $G$ through $x$.  Our lifting layer thus has the advantage of being not only trainable so that it may extract features, but also of being more memory-efficient, producing only a doublet rather than a triplet after its forward pass.

Finally, notice that in Theorem \ref{firstlayerconv}, even though $f$ and $\psi$ are compactly supported, their convolution need not be since the action of $G$ need not be proper (see footnote 2).  The presence of the $\det(Du^{-1}(x))$ in Equation \eqref{eq6} moreover means that $f*\psi$ need not even be bounded (consider for instance the group of scaling transformations).  We show in the next section that these technicalities do not prevent group convolutions in higher layers.

\subsection{Hidden convolutional layers}\label{hiddenlayer}

Having lifted an image to a continuous function on $G$ using the first layer detailed in the previous subsection, we must now implement convolutions over $G$.  However, as alluded to in the previous section, the feature maps produced by Theorem \ref{firstlayerconv} may be unbounded, and the ordinary theory \cite{gpequi, kondortrivedi, renault} therefore does not apply.  Showing that convolutions are possible even in the unbounded case is our second theoretical contribution; we defer a proof to the supplementary material.

\begin{thm}\label{conv}
Define the convolution of a feature map $f\in C(G;\RB^{K})$ with a filter map $\psi\in C_{c}(G;\RB^{K\times L})$ by the formula
\begin{align}
f*\psi(u):=&\int_{G}f(v)\cdot\psi(v^{-1}u)d\mu_{L}(v)\label{convleft}\\ =& \int_{G}f(uv^{-1})\cdot\psi(v)d\mu_{R}(v)\label{convright},
\end{align}
where $\cdot$ denotes matrix multiplication.  Then $f*\psi$ is a continuous function on $G$, and for all $v\in G$ one has $L_{v}(f)*\psi = L_{v}(f*\psi).$\qed
\end{thm}

The next task is to describe how convolutions may be implemented.  In \cite{gpequi}, this is achieved by assuming $G$ to be a discrete group, so that the values of any feature map at each group element can be encoded in channels of a corresponding feature tensor.  While efficient computationally, this approach is insufficient for a continuous Lie group such as the affine or homography group.

In \cite{surjexp}, Equation \eqref{convleft} is used in approximating convolutions using a Monte-Carlo method.  Importantly, the method adopted in \cite{surjexp} makes the following two assumptions.
\begin{itemize}
    \item \textbf{Assumption 1 of \cite{surjexp}}: It is assumed that $G$ has a surjective exponential map, so that filters can be parameterised by multi-layer perceptrons (MLPs), which take \emph{vectors} (that is, Lie algebra elements) as input.
    \item \textbf{Assumption 2 of \cite{surjexp}}: It is assumed that Haar measure on $G$ is easily reducible to known measures on known sets (such as uniform measures on Euclidean space or the unit quaternions). Such measures are easily sampled from so that a Monte-Carlo estimate of Equation \eqref{convleft} can be obtained.
\end{itemize}
For the affine and homography groups, neither Assumption 1 nor Assumption 2 holds.  Our next theorem allows us to overcome these assumptions entirely. 

\begin{thm}\label{coordchange}
If $f$ is an integrable function on $G$ which is zero outside of a sufficiently small neighbourhood of the identity, then Haar measure $d\mu_{R}$ can always be chosen such that
\begin{equation}\label{eqp4}
\int_{G}f(u)d\mu_{R}(u) = \int_{\gf}f(\exp(\xi))\,\det\bigg(\frac{1-e^{-\ad_{-\xi}}}{\ad_{-\xi}}\bigg)d\xi,
\end{equation}
where $d\xi$ denotes a Euclidean volume element in the vector space $\gf$, and $\ad_{\xi}:\gf\rightarrow\gf$ and $(1-e^{-\ad_{\xi}})/\ad_{\xi}$ are given in Theorem \ref{dexp}.
\end{thm}

A rigourous proof of the theorem requires some differential geometry (differential forms, their integrals and their pullbacks), which we defer to the supplementary material.  For now we just give an idea of the proof.

\begin{proof}[Idea of proof]
The small support of $f$ enables one to use a change of coordinates $u=\exp(\xi)$, reducing the integral over $G$ to an integral over $\gf$.  The change of coordinates is implemented by the exponential map, and therefore the determinant of its derivative (Theorem \ref{dexp}) appears as a multiplicative factor to keep track of the way volumes are distorted by this change of coordinates.  This explains the presence of the power series $(1-e^{-\ad_{-\xi}})/\ad_{-\xi}$ in the new expression.  That multiplication by $\exp(\xi)$ does not appear in this new expression, despite its appearance in Theorem \ref{dexp}, follows from the invariance of Haar measure.
\end{proof}

We now combine Theorem \ref{coordchange} with Equation \eqref{convright}, and demonstrate how they allow us to circumvent the assumptions made in \cite{surjexp}.  Fix a feature map $f\in C(G;\RB^{K})$, and compactly supported filter $\psi\in C_{c}(G;\RB^{K\times L})$. Assume $\psi$ to have support in a sufficiently small neighbourhood of the identity that $\tilde{\psi}:=\psi\circ\exp^{-1}$ is well-defined. Then the convolution $f*\psi$ is given by
\begin{equation}\label{eq8}
\int_{\gf}f(u\exp(-\xi))\cdot\tilde{\psi}(\xi)\det\bigg(\frac{1-e^{-\ad_{-\xi}}}{\ad_{-\xi}}\bigg)d\xi.
\end{equation}
Let us now explain how Equation \eqref{eq8} allows us to overcome Assumptions 1 and 2.
\begin{itemize}
    \item \textbf{Overcoming assumption 1 of \cite{surjexp}}: The replacement of $\psi$ with $\tilde{\psi}$ in Equation \eqref{eq8} is key for implementation, as it allows trainable filters to be parameterised by small neural networks which take vectors (Lie algebra elements) as input rather than group elements.  Equation \eqref{eq8} holds for \emph{any Lie group}.  In contrast, Equation \eqref{convleft} is used for the convolution $f*\psi$ in \cite{surjexp}, meaning the filter $\psi$ needs to take group elements $v^{-1}u$ as input.  Since these group elements may be arbitrarily far from the identity, surjectivity of the exponential map must be assumed in order to write $\psi = \tilde{\psi}\circ\exp^{-1}$.
    \item \textbf{Overcoming assumption 2 of \cite{surjexp}}:  There is a large family of Markov chain Monte-Carlo (MCMC) methods for sampling from distributions (measures) on Euclidean spaces \cite{mcmc}.  Such methods do not, however, immediately apply to non-Euclidean spaces such as Lie groups.  In \cite{surjexp}, sampling from Haar measure on a Lie group is instead achieved by identifying Haar measure with more well-known measures which are easily sampled from.  In contrast, using Equation \eqref{eq8} allows us to use MCMC methods to sample from Haar measure for \emph{any} Lie group.  Indeed, the Lie algebra of any Lie group is a Euclidean space, and the function
    \[
    \xi\mapsto\det\bigg(\frac{1-e^{-\ad_{-\xi}}}{\ad_{-\xi}}\bigg)
    \]
   thereon can then be thought of as (being proportional to) a density function.  With this density function in hand, standard MCMC methods such as the Metropolis algorithm may be utilised in order to sample from Haar measure to produce a Monte-Carlo estimate
   \[
   f*\psi(u)\approx \frac{1}{N}\sum_{\xi_{i}\sim\exp^{*}\mu_{R}}f(u\exp(-\xi_{i}))\cdot\tilde{\psi}(\xi_{i})
   \]
   of the convolution $f*\psi$.
\end{itemize}

We remark that MCMC methods like the Metropolis algorithm are easily parallelised \cite{parallelmetro1, parallelmetro2}.  For the fast estimation of convolutions, we parallelise the Metropolis algorithm to produce multiple samples from Haar measure at once.  We achieve this by running multiple chains at once, using boolean matrices to keep track of which samples are accepted and which are rejected.  This parallelisation results in an order of magnitude greater speed than is obtained via the conventional method of producing samples sequentially.

\subsection{Nonlinearities and fully connected hidden layers}

It was observed in \cite{gpequi}, and is easily verified, that if $f:G\rightarrow\RB^{K}$ is any continuous function, then postcomposition with any function $\phi:\RB^{K}\rightarrow\RB^{L}$ is an equivariant operation:
\[
\phi\circ L_{v}(f) = L_{v}(\phi\circ f),\qquad v\in G.
\]
Therefore conventional nonlinearities may be applied in the usual way after any convolutional layer without breaking equivariance, and fully connected layers can be stacked on top of convolutional layers while preserving equivariance.

\subsection{Final layer}\label{finallayer}

The final layer is a pooling layer. Since our feature maps are not generally bounded, \emph{global} pooling need not be well-defined.  However, our feature maps are continuous, meaning their maxima over any compact set are well-defined.  Assuming a $C$-class classification problem, given a continuous function $f\in C(G;\RB^{K})$, a matrix $A\in\RB^{K\times C}$ and a bias vector $b\in\RB^{C}$, the componentwise operation
\[
f\mapsto\max_{u\in K}(Af(u) +b)
\]
is well-defined for any compact subset $K$ of $G$.  Our next and final theorem, proved in the supplementary material, gives a theoretical guarantee of invariance of this operation.

\begin{thm}\label{maxthm}
Let $f:G\rightarrow\RB$ be a (possibly unbounded) continuous function, and $K\subset G$ a compact neighbourhood of the identity.  If $f$ takes its maximum value over $K$ only at points in the interior of $K$, then there is an open neighbourhood $V$ of the identity for which $\max_{u\in K}L_{v}(f)(u) = \max_{u\in K}f(u)$ for all $v\in V$.\qed
\end{thm}

\section{Experiments}

Our framework enables $G$-invariant models, which we test by training on MNIST and testing them on $G$-perturbed MNIST.  We take $G$ to be the affine and homography groups.  Testing on affine-perturbed MNIST (affNIST) after training on MNIST is the benchmark in affine-invariant image classification \cite{capsule2, sparsecaps, gecaps, affcapsnets}, while the analogous test for homography-invariant image classification has not yet been studied. To test homography-invariance, we created a \emph{homNIST}\footnote{https://www.kaggle.com/datasets/lachlanemacdonald/homnist} test set by sampling 32 homographies from right Haar measure for each of the MNIST test images.

In both cases, our networks consist of a lifting convolutional layer (Subsection \ref{firstlayer}), a single $G$-convolutional layer (Subsection \ref{hiddenlayer}), two residual fully connected layers and finally a pooling layer (Subsection \ref{finallayer}).  In each forward pass, a new collection of pooling samples is drawn from $G$ according to a Gaussian distribution centered at the identity, and a new collection of samples for estimating the convolutional integral is drawn from right Haar measure using the Metropolis MCMC algorithm.  Models were trained for 150 epochs with a batch size of 60, using Adam with the recommended default settings \cite{adam}.  The tables in Figure \ref{tables} benchmark our models' mean test performances, taken at the final epoch, over five training runs from different initialisations, with standard errors.  We were unfortunately unable to locate code or parameter counts for the affine CapsNet of \cite{affcapsnets}, and those of \cite{routinguncertainty} have been estimated from a smaller model in \cite{routinguncertainty} for which parameter counts are given.  We predict that other CapsNets would exhibit similar performance drops on homNIST relative to affNIST as the other CapsNets.  Curiously, the $E(2)$-equivariant CNN of \cite{steerable4} (which we trained using the same regime as the authors, albeit on MNIST only) achieves worse performance on homNIST than on affNIST: we expect this is due to homNIST exhibiting smaller affine perturbations than affNIST, which was necessary to rule out edge effects\footnote{Quoted figures are means with standard errors over four training runs with different random seeds}.

\begin{figure}
\begin{center}
\begin{tabular}{||c c c ||} 
 \hline
 Model & affNIST Test Acc. & Parameters \\ [0.5ex] 
 \hline\hline
  RU CapsNet \cite{routinguncertainty} & 97.69 & $>$580K\\
 \hline
 \bf{affConv (ours)} & \bf{95.08($\pm$0.31)} & \bf{373K} \\ 
 \hline
 affine CapsNet \cite{affcapsnets} & 93.21 &  \\
 \hline
 GE CapsNet \cite{gecaps} & 89.10 & 235K \\
 \hline
 CapsNet \cite{capsule2} & 79 & 8.1M \\
 \hline
 CNN \cite{capsule2} & 66 & 35.4M \\
 \hline
  E2SFCNN \cite{steerable4} & 57.10($\pm$ 0.71) & 7.07M \\
 \hline
\end{tabular}
\end{center}

\begin{center}
\begin{tabular}{||c c c ||} 
 \hline
 Model & homNIST Test Acc. & Parameters \\ [0.5ex] 
 \hline\hline
 \bf{homConv (ours)} & \bf{95.71($\pm$0.09)} & \bf{376K}\\
 \hline
 GE CapsNet \cite{gecaps} & 84.67  & 235K \\
 \hline
  E2SFCNN \cite{steerable4} & 82.53($\pm$ 0.30) & 7.07M \\
 \hline
 CapsNet \cite{capsule2} & 74.95 & 8.1M \\
 \hline
\end{tabular}
\end{center}
\caption{\small Test performance comparison on affNIST and homNIST after training on MNIST.}
\label{tables}
\end{figure}

We predict performance to degrade as our approximations to idealised convolutions and max pools get worse.  To test this prediction, we train our models using the same initialisation and different numbers $N$ of samples to approximate convolution and pooling, and evaluate test performance against $N$.  Figure \ref{numsamples} shows degradation of all test performances as the number of samples is decreased from $N=100$ to $N=25$, until catastrophic failure at $N=1$.

\begin{figure}
     \centering
     \begin{subfigure}[b]{0.455\textwidth}
        \small
         \begin{tabular}{|c c | c c c c c|}
            \cline{3-7}
            \multicolumn{1}{c}{} & & \multicolumn{5}{|c|}{Num. pool samples} \\
            \hline
            \multicolumn{1}{|c}{Model} & \multicolumn{1}{c|}{Data} & 100 & 75 & 50 & 25 & 1 \\
            \hline\hline
            \multirow{2}{*}{affConv} & MNIST & 98.6 & 98.5 & 98.4 & 97.9 & 68.1 \\
             & affNIST & 94.6 & 94.7 & 93.6 & 91.5 & 51.1 \\
            \hline
            \multirow{2}{*}{homConv} & MNIST & 98.7 & 98.7 & 98.4 & 98.5 & 78.6 \\
             & homNIST & 96.0 & 95.1 & 95.4 & 94.3 & 66.2\\ \hline
        \end{tabular}
        \caption{Test accuracy versus number of pooling samples}
     \end{subfigure}
     \par\bigskip
     \begin{subfigure}[b]{0.455\textwidth}
     \small
        \begin{tabular}{|c c | c c c c c|}
            \cline{3-7}
            \multicolumn{1}{c}{} & & \multicolumn{5}{|c|}{Num. conv. samples} \\
            \hline
            \multicolumn{1}{|c}{Model} & \multicolumn{1}{c|}{Data} & 100 & 75 & 50 & 25 & 1 \\
            \hline\hline
            \multirow{2}{*}{affConv} & MNIST & 98.3 & 98.2 & 98.1 & 96.3 & 36.0 \\
             & affNIST & 94.1 & 93.3 & 92.5 & 88.0 & 30.0 \\
            \hline
            \multirow{2}{*}{homConv} & MNIST & 98.9 & 98.7 & 98.4 & 97.2 & 41.0 \\
             & homNIST & 95.7 & 95.7 & 94.7 & 92.2 & 36.3\\ \hline
        \end{tabular}
        \caption{Test accuracy versus number of convolutional samples}
        \label{Nconv}
     \end{subfigure}
        \caption{\small Test accuracies are relatively stable for smaller numbers of samples until catastrophic failure with only one sample.  This failure is much more pronounced when convolutional samples are decreased than when pooling samples are decreased.  Each test value is from a single training run.}
        \label{numsamples}
\end{figure}

We tested the equivariance error of our layers, measured in the same manner as in \cite[Section 6.1]{scaleequi}.  The results for the first layer are shown in Figure \ref{equivarianceerror}.  Interestingly, the equivariance error of higher convolutional layers is \emph{always zero}, regardless of the number of samples used to approximate the convolutions.  This can be seen mathematically: let $V:=\{v_{i}\}_{i=1}^{N}$ be any finite sample from Haar measure on a group $G$, let $f\in C(G)$ be a feature map, with $\psi\in C_{c}(G)$ a filter map.  Define $f*_{V}\psi\in C(G)$ by the formula
\[
f*_{V}\psi(u):=\frac{1}{N}\sum_{i=1}^{N}f(uv_{i}^{-1})\psi(v_{i}).
\]
Monte-Carlo theory tells us that as the size of the sample $V$ grows, the error in approximating the true, equivariant convolution $f*\psi$ by the Monte-Carlo estimate $f*_{V}\psi$ shrinks.  However, for any $w\in G$, one has perfect equivariance
\begin{align*}
L_{w}(f)*_{V}\psi(u) =& \sum_{i=1}^{N}L_{w}(f)(uv_{i}^{-1})\psi(v_{i})\\ =& \sum_{i=1}^{N}f(w^{-1}uv_{i}^{-1})\psi(v_{i}) \\ =& L_{w}(f*_{V}\psi)(u)
\end{align*}
\emph{regardless of the size of $V$}.  This analysis suggests the degradation in Figure \ref{numsamples} is a consequence of the higher variance of the forward pass arising from the paucity of samples, resulting in a lack of trainability, rather than being a result of a lack of equivariance.

\begin{figure}
    \centering
        \begin{tabular}{|c c |}
            \hline
            \multicolumn{1}{|c}{Model} & \multicolumn{1}{c|}{Equivariance error}  \\
            \hline\hline
            affConv & 5.52\% \\
            \hline
            homConv & 6.86\% \\
            \hline
        \end{tabular}
        \caption{\small Equivariance errors $\|L_{u}(f*\psi)(v) - L_{u}(f)*\psi(v)\|^{2}_{2}/\|L_{u}(f*\psi)(v)\|^{2}_{2}$ of the first layer, averaged over 100 sample warps $u$ and 100 group elements $v$.}
        \label{equivarianceerror}
\end{figure}

\section{Discussion}

\noindent\textbf{On theoretical guarantees:}  Despite their desirability for building trustworthy AI, theoretical guarantees of performance are still uncommon in modern machine learning.  In addition to the greater generality of our proposed framework, one advantage of our approach to invariance over conventional CapsNets is the availability of rigourous mathematical guarantees of invariance of \emph{idealised models}.  In practice these guarantees do not apply exactly due to the approximate methods used in implementation. Providing rigourous mathematical guarantees for \emph{practical} models, which take into account the stochastic methods used to approximate continuous integrals and max pools, is desirable.

\noindent\textbf{On depth:}  One limitation of the framework we propose is that naively adding more convolutional layers results in a corresponding increase in memory usage beyond that required to store the additional weights.  If nonlinearities are not used between convolutional layers, this increase in memory usage can be made \emph{linear} in the number of layers using well-known path-integral methods \cite{mcnonlinear}, however this sacrifices expressivity.  In a naive implementation, keeping nonlinearities between convolutional layers implies an exponential growth in memory usage with the number of convolutional layers -- with $L$ convolutional layers, and $N$ samples used to approximate each convolution, one needs $N^{L}$ function evaluations in total\footnote{Combinatorial explosion is also present in \cite[Section 4.4]{surjexp}, where for each point $u\in G$, a \emph{new} family of samples $\{v_{i}\}\in\text{nbhd}(u)$, dependent on $u$, must be drawn from Haar measure to estimate the convolution $f*\psi(u)$.  Thus for instance to estimate $f*\psi_{1}*\psi_{2}(u)$ naively, one requires samples $\{v_{i}\}\in\text{nbhd}(u)$ for the outermost integral and, for each $i$, new samples $\{w_{i,j}\}\in\text{nbhd}(v_{i})$ dependent on $v_{i}$ for the innermost integral.  In practice, the authors of \cite{surjexp} avoid this combinatorial explosion by using the \emph{same} samples for each layer, however they provide no guarantee that this sampling method yields an unbiased estimate of the true convolution.}.  Practically, this means very large memory usage at the front end of the forward pass. However, recent work on nonlinearities in Monte-Carlo methods \cite{mcnonlinear} together with our ablation (Figure \ref{numsamples}) suggest that this memory usage can be made vastly less while retaining good performance.  Moreover, our empirical and theoretical observation in the Experiments section that higher convolutional layers are always perfectly equivariant, regardless of convolutional sample size, holds promise for use in more scalable architectures.

\section{Conclusion}

We introduced a mathematical framework for convolutional neural networks over \emph{any} Lie group $G$.  We have given rigourous guarantees of equi/invariance of idealised models, and a new, parallelisable method for sampling from Haar measure for arbitrary Lie groups.  We have empirically validated our framework by using to achieve competitive results on benchmark invariance tasks in image classification.

\section{Appendix}

Here we collect rigourous mathematical proofs of the claimed theorems.  To follow the proofs of Theorems \ref{firstlayerconv}, \ref{conv} and \ref{maxthm}, the reader will need to know some topological definitions; specifically of topological spaces and open sets, compact and locally compact Hausdorff spaces,   and continuity of maps; for these we refer the reader to \cite{munkres}.  To follow the proof of Theorem \ref{coordchange}, the reader will need to be familiar with differential forms on manifolds, their pullbacks and their integrals; for these we refer the reader to \cite{leesmth}.

In order to prove Theorems \ref{firstlayerconv} and \ref{conv}, we require the following technical lemma.

\begin{lemma}\label{lemma}
Let $X$ and $Y$ be locally compact Hausdorff spaces, and let $F:X\times Y\rightarrow\RB^{n}$ be a continuous function for which there exists a compact set $K\subset Y$ such that the support of $F$ is contained in $X\times K$.  If $\nu$ is a Radon measure on $Y$, then the map
\begin{equation}\label{lemmaeq}
X\ni x\mapsto \int_{Y}F(x,y)d\nu(y)\in\RB^{n}
\end{equation}
is a continuous function on $X$.
\end{lemma}

\begin{proof}
The proof is almost identical to \cite[Lemma 1.102]{williams}.  In \cite[Lemma 1.102]{williams}, $Y$ is instead a locally compact Hausdorff group, with $\nu$ being Haar measure.  However the group properties are not required for the proof of continuity.  The only other difference is that in \cite[Lemma 1.102]{williams}, $F$ is assumed to have compact support in both the $x$ and $y$ variables, while we hypothesise only compact support in the $y$ variable.  This difference also does not impact the proof that Equation \eqref{lemmaeq} defines a continuous map.
\end{proof}

\begin{proof}[Proof of Theorem \ref{firstlayerconv}]
In fact the theorem holds in much greater generality.  Let $X$ be any Hausdorff manifold, $dx$ a volume form on $X$, and suppose that $G$ acts smoothly on $X$.  Let $f\in C_{c}(X;\RB^{K})$ and $\psi\in C_{c}(X,\RB^{K\times L})$.  Then we must show that the map $u\mapsto \int_{X}f(x)\cdot\psi(u^{-1}\cdot x)\det(Du^{-1}(x))dx$ is continuous as a function $G\rightarrow\RB^{L}$.  The function $G\times X\ni(u,x)\mapsto \psi(u^{-1}\cdot x)\det(Du^{-1}(x))\in\RB^{K\times L}$ is continuous, by continuity of $\psi$ and by smoothness of the action of $G$ on $X$.  It follows then that $(u,x)\mapsto f(x)\cdot\psi(u^{-1}\cdot x)\det(Du^{-1}(x))$ is continuous on $G\times X$ with support in $G\times C$, where $C$ is the compact support of $f$. Lemma \ref{lemma} then implies that $u\mapsto\int_{X}f(x)\cdot\psi(u^{-1}\cdot x)\det(Du^{-1}(x))dx$ is a continuous function on $G$.
\end{proof}

\begin{proof}[Proof of Theorem \ref{conv}]
Recall that we are given $f\in C(G;\RB^{K})$ and $\psi\in C_{c}(G;\RB^{K\times L})$.  We wish to show that $f*\psi$ defined by
\[
f*\psi(u):=\int_{G}f(uv^{-1})\cdot\psi(v)d\mu_{R}(v)
\]
is continuous as a function on $G$.  By continuity of $f$ together with continuity of the multiplication in $G$, the function $(u,v)\mapsto f(uv^{-1})$ is continuous as a map $G\times G\rightarrow\RB^{K}$.  Let $C$ be the compact support of $\psi$.  We then see that $(u,v)\mapsto f(uv^{-1})\cdot\psi(v)$ is a continuous function $G\times G\rightarrow\RB^{L}$, with support contained in the set $G\times C$.  Lemma \ref{lemma} then implies that $u\mapsto \int_{G}f(uv^{-1})\cdot\psi(v)d\mu_{R}(v)$ is continuous as a function on $G$.  The fact that
\[
\int_{G}f(uv^{-1})\cdot\psi(v)d\mu_{R}(v) = \int_{G}f(w)\cdot\psi(w^{-1}u)d\mu_{L}(w)
\]
follows from making the substitution $w = uv^{-1}$ and invoking the identity $d\mu_{R}(w^{-1}u) = d\mu_{L}(u^{-1}w)$, followed by the left-invariance $d\mu_{L}(u^{-1}w) = d\mu_{L}(w)$ of $\mu_{L}$.
\end{proof}

\begin{proof}[Proof of Theorem \ref{coordchange}]
We begin by constructing a \emph{left} Haar measure $\mu_{L}$ for which an analogous formula holds, and then the result follows from the identity $d\mu_{L}(u) = d\mu_{R}(u^{-1})$.  Let $d\xi$ be the standard Euclidean volume element in the vector space $\gf$ defined with respect to some choice of basis.  Thus $d\xi = d\xi^{1}\wedge\cdots \wedge d\xi^{n}$ where $(\xi^{1},\dots,\xi^{n})$ are the coordinates defined by the basis.  The cotangent multi-vector $d\xi|_{0}$ obtained by evaluating the form $d\xi$ at $0\in\gf$ is then a nonzero volume element at the identity of $G$.  Now the pullback formula
\[
(d\mu_{L})_{u}:=L_{u^{-1}}^{*}d\xi|_{0}
\]
defines a top-degree form on $G$, which is left-invariant since
\begin{align*}
(L_{v}^{*}d\mu_{L})_{u} =& L_{v}^{*}(d\mu_{L})_{vu} = L_{v}^{*}L_{u^{-1}v^{-1}}^{*}d\xi|_{0}\\ =& (L_{u^{-1}v^{-1}}\circ L_{v})^{*}d\xi|_{0} = L_{u^{-1}}^{*}d\xi|_{0} = (d\mu_{L})_{u}
\end{align*}
for any $v\in G$.  Thus $d\mu_{L}$ is a left Haar measure.

Now if $f$ is an integrable function which is zero outside of a neighbourhood of the identity onto which the exponential map is a diffeomorphism, then we have
\[
\int_{G}f(u)d\mu_{L}(u) = \int_{\gf}f(\exp(\xi))(\exp^{*}d\mu_{L})(\xi),
\]
so we must compute $\exp^{*}d\mu_{L}$.  Fix $\xi\in\gf$, and let $\{t\mapsto \xi_{i}(t) = \xi+t\xi^{i}\}_{i=1}^{n}$ be the curves in $\gf$ through $\xi$ pointing in the $n$ coordinate directions.  Let $p(\xi)$ be the power series $(1-e^{-\ad_{\xi}})/\ad_{\xi}$.  Then we compute
\begin{align*}
    & \big(\exp^{*}d\mu_{L}\big)_{\xi}\big(\xi_{1}'(0)\wedge\cdots\wedge\xi_{n}'(0)\big)\\ =& (d\mu_{L})_{\exp(\xi)}\big(\exp(\xi_{1})'(0)\wedge\cdots\wedge\exp(\xi_{n})'(0)\big) \\ =& L_{\exp(-\xi)}^{*}d\xi|_{0}\big(L_{\exp(\xi)}p(\xi)\xi^{1}\wedge\cdots\wedge L_{\exp(\xi)}p(\xi)\xi^{n}\big)\\ =& d\xi|_{0}\big(p(\xi)\xi^{1}\wedge\cdots\wedge p(\xi)\xi^{n}\big) = \det(p(\xi)).
\end{align*}
Here, the first equality follows from the definition of the pullback, the second from Theorem \ref{dexp}, the third from the definition of $d\mu_{L}$ and the final from the fact that the top-degree wedge product of a linear map is equal to its determinant, together with the identity $d\xi|_{0}(\xi^{1}\wedge\cdots\wedge\xi^{n}) = d\xi^{1}(\xi^{1})\cdots d\xi^{n}(\xi^{n})=1$.  It follows that
\[
\exp^{*}d\mu_{L}(\xi) = \det\bigg(\frac{1-e^{-\ad_{\xi}}}{\ad_{\xi}}\bigg)d\xi,
\]
hence
\[
\exp^{*}d\mu_{R}(\xi) = \det\bigg(\frac{1-e^{-\ad_{-\xi}}}{\ad_{-\xi}}\bigg)d\xi
\]
follows from the identities $d\mu_{R}(u) = d\mu_{L}(u^{-1})$ and $\exp(\xi)^{-1} = \exp(-\xi)$.
\end{proof}

To prove Theorem \ref{maxthm}, we require the following lemma.

\begin{lemma}\cite[Proposition 3]{nachbin}\label{nachbin}
Let $\Lambda$ be a compact topological space, and let $\{V_{\lambda}\}_{\lambda\in\Lambda}$ be a family of open subsets of some other topological space $X$.  Then the intersection $\bigcap_{\lambda\in\Lambda}V_{\lambda}$ is an open set in $X$.\qed
\end{lemma}

\begin{proof}[Proof of Theorem \ref{maxthm}]
Let $M$ denote the maximum value of $f$ over $K$.  We will prove the theorem by showing that there exist an open neighbourhood $V_{1}$ of the identity such that $\max_{u\in K}f(v^{-1}u)\leq M$ for all $v\in V_{1}$, and an open neighbourhood $V_{2}$ of the identity for which $\max_{u\in K}f(v^{-1}u)\geq M$ for all $v\in V_{2}$.  Then $V:=V_{1}\cap V_{2}$ will be an open neighbourhood of the identity for which $\max_{u\in K}f(v^{-1}u) = M$ for all $v\in V$.  In showing the existence of both $V_{1}$ and $V_{2}$ we will use the following observation. For each $u\in G$, let $\alpha_{u}:G\rightarrow G$ denote the map $v\mapsto v^{-1}u$. Then by continuity of the group operations, each $\alpha_{u}$ is continuous.  Therefore, letting $K^{\circ}$ denote the interior of $K$ (an open set), the pre-image $\alpha_{u}^{-1}K^{\circ}$ of $K^{\circ}$ under $\alpha_{u}$ is open for every $u\in G$.  In particular, since $f^{-1}\{M\}\cap K$ is contained in $K^{\circ}$ by hypothesis, for all $u\in f^{-1}\{M\}\cap K$ the set $\alpha_{u}^{-1}K^{\circ}$ is an open neighbourhood of the identity.

Existence of $V_{1}$ now follows easily.  For each $u\in f^{-1}\{M\}\cap K$, let $V_{1,u}:=\alpha_{u}^{-1}K^{\circ}$ be the open neighbourhood of the identity considered in the previous paragraph.  The set $f^{-1}\{M\}\cap K$ is a closed subset of a compact set, hence compact.  Thus, by Lemma \ref{nachbin}, $V_{1}:=\bigcap_{u\in f^{-1}\{M\}\cap K}V_{1,u}$ is an open neighbourhood of the identity such that $v^{-1}u\in K$ for all $v\in V_{1}$ and $u\in f^{-1}\{M\}\cap K$.  It follows that $\max_{u\in K}f(v^{-1}u)\geq M$ for all $v\in V_{1}$.

We now come to showing the existence of $V_{2}$, which will be an intersection over $u\in K$ of open neighbourhoods $V_{2,u}$ of the identity, for which $v^{-1}u\in f^{-1}(-\infty,M]$ for all $v\in V_{2,u}$ and $u\in K$.  Write $K$ as the union $K_{1}\cup K_{2}$, where $K_{1}:=K\cap f^{-1}(-\infty,M)$ and $K_{2}:=K\cap f^{-1}\{M\}$.  For all $u\in K_{1}$, the continuity of $\alpha_{u}$ implies that $V_{1,u}:=\alpha_{u}^{-1}f^{-1}(-\infty,M)\subset\alpha_{u}^{-1}f^{-1}(-\infty,M]$ is an open set, which contains the identity since $u\in K_{1}$.  For $u\in K_{2}$, we take $V_{2,u}$ to be the open neighbourhood $\alpha_{u}^{-1}K^{\circ}\subset \alpha_{u}^{-1}f^{-1}(-\infty,M]$ of the identity described in the first paragraph.  Now by Lemma \ref{nachbin}, $V_{2}:=\bigcap_{u\in K}V_{2,u}$ is an open neighbourhood of the identity such that $\max_{u\in K}f(v^{-1}u)\leq M$ for all $v\in V_{2}$, as required.
\end{proof}

{\small
\bibliographystyle{ieee_fullname}
\bibliography{egbib}
}

\end{document}